\documentclass[12pt]{article}

\usepackage[text={16cm,23cm},centering]{geometry}
\usepackage[numbers]{natbib}
\usepackage{algorithm}
\usepackage{algorithmic}
\usepackage{amsmath}
\usepackage{amsthm}
\usepackage{verbatim}
\usepackage{epstopdf}

\newtheorem{theorem}{Theorem}[section]
\newtheorem{lemma}[theorem]{Lemma}
\newtheorem{corollary}[theorem]{Corollary}

\begin{document}

\title{Learning Network Structures from Contagion}
\author{
	Adisak Supeesun and Jittat Fakcharoenphol \\
	\footnotesize Department of Computer Engineering, Kasetsart University, Bangkok, Thailand, 10900 \\
	\footnotesize \texttt{\{adisak.sup, jittat\}@gmail.com}
	\footnotesize \date{}
}
\maketitle

\begin{abstract}
In 2014, Amin, Heidari, and Kearns proved that tree networks 
can be learned by observing only the infected set of vertices of the
contagion process under the independent cascade model, in both the
active and passive query models.  They also showed empirically that
simple extensions of their algorithms work on sparse networks.  In this
work, we focus on the active model.  We prove that a simple
modification of Amin {\em et al.}'s algorithm works on more general
classes of networks, namely (i) networks with large girth and low path
growth rate, and (ii) networks with bounded degree.  This also
provides partial theoretical explanation for Amin {\em et al.}'s
experiments on sparse networks.
\end{abstract}


\section{Introduction}
\label{section:intro}
Edge information of a network is essential to many network analysis
applications, e.g., in social network analysis~\cite{G1973, TSWZ2009,
  ST2011}, in epidemiology~\cite{CPBCFBT2005}, and in viral
marketing~\cite{LAH2006}.  However, in some type of networks, e.g.,
online social networks or the network of customers of a rival company,
it is not easy to obtain this information.  Therefore, in these
settings, it is more practical to infer the network structures from
observable data.

This paper considers a problem of inferring the network structure from
a contagion process, known as the {\em independent cascade model}
(defined by Goldenberg {\em et al.}~\cite{GLM2001a, GLM2001b} and
Kempe {\em et al.}~\cite{KKT2003}).  The problem was first considered
by Gomez-Rodriguez, Leskovec, and Krause~\cite{GLK2010}.  While most
results utilize the orders of infections to infer the network
structures (e.g.,~\cite{GLK2010, ML2010, GS2012, NS2012, ACKP2013}),
the set of infected vertices are clearly easier to observe.  Thus, we
follow a recent work by Amin, Heidari, and Kearns~\cite{AHK2014} who
introduced a problem of learning unknown network structure by
observing only sets of initial infected vertices and sets of final
infected vertices.  Amin {\em et al.} considered both the {\em active
  model} where the algorithm can make active seed queries and the {\em
  passive model} where the algorithm only observes the seed and the
set of infected vertices, and proposed algorithms for exactly learning
tree networks under these two models. They also proposed another
algorithm for learning general networks, called the $K$-lifts
algorithm, that works well empirically under real networks and random
networks.  Amin {\em et al.} proved that the $K$-lifts algorithm can
learn cycle networks and provided an example which is a union of a
star and a large cycle that the algorithm fails to learn.
 
In this paper, we consider the active model and extend the approach of
Amin {\em et al.}~\cite{AHK2014}, to two other classes of networks, as
described below.

{\bf Networks with large girth and low path growth rate.}
Essentially, these are networks that behave almost like trees.  We
consider two parameters of the networks: (1) the girth, which is the
size of the smallest cycle in the network and (2) the growth rate on
the number of paths between two vertices over the length of the paths.
We show that if the girth of a network is large enough and the growth
rate is small enough, the network can be exactly learned by an
algorithm that uses only a polynomial number of queries in term of the
number of vertices and $\Delta$, the contagion parameter (to be
defined in the next section).  Since a tree does not contain any
cycle, its girth is $\infty$; this class of networks generalizes the
tree networks considered in~\cite{AHK2014}.  This class also includes
a counter example for the $K$-lifts algorithm provided
in~\cite{AHK2014}.  While the focus of this work is extending the
theoretical limitation of Amin {\em et al.}~\cite{AHK2014}'s work,
this type of networks may appear in advisor-advisee networks where
cross-field advising happens on rare occasions or in organizational
networks, networks that represent ranks and relations of people in
organizations, where few low-ranking workers may report to more than
one middle managers, resulting in large cycles that span the people at
the top-most levels of the organizations.

{\bf Bounded degree networks.} We also show that if the maximum degree
of vertices in the network is bounded above by a constant, the network
can be exactly learned with polynomial queries as well.

The paper is organized as follows.  In the next section, we state
problem definitions and discuss previous results.
Section~\ref{section:large_girth} shows that a large girth network
with low path growth can be learned in the contagion model.  In
section \ref{section:bounded_degree}, we show that if the maximum
degree of a network is bounded by a constant, the network can also be
learned.

\section{Definitions and Results}
\label{section:problem}

We formally describe the contagion process.  Let $G=(V,E)$ be an
undirected network whose edges are unknown with $n$ vertices.  Let
$S\subseteq V$ be the {\em seed set}, the set of initially infected
vertices.  From $S$, a contagion proceeds in discrete steps under the
independent cascade model defined by Goldenberg {\em et
  al.}~\cite{GLM2001a, GLM2001b} and Kempe {\em et
  al.}~\cite{KKT2003}, as follows.  We assume that every vertex in $S$
becomes infected at step $t=0$.  At each step $t=1,2,3,\ldots$, every
vertex $u \in V$ which became infected at step $t-1$ tosses a coin to
spread the disease to each of its uninfected adjacent vertices $v \in
V$ with the infectious probability $p_{uv}$.  If $v$ receives the
disease from $u$, we say that \textit{$v$ becomes infected at step
  $t$}.  In this case, we say that the edge $(u,v)$ is
\textit{active}, otherwise $(u,v)$ is \textit{inactive}.  Note that
when referring to an edge as active or inactive, the order of its end
points in the tuple is important (e.g., when $(u,v)$ is active,
$(v,u)$ is inactive).  The contagion proceeds until there are no newly
infected vertices.  Note that spreading of disease through edge
$(u,v)$ occurs only once at the first step when $u$ or $v$ become
infected.  The minimum and maximum infectious probabilities are
denoted by $\alpha$ and $\beta$, respectively.  We define the
contagion parameter $\Delta = \min\{\alpha, 1-\beta\}$.

The problem of learning network structure from contagion is defined as
follows.  For a network $G=(V,E)$, we are given the set of vertices
$V$ and the contagion parameter $\Delta$, but for all edges $(u,v) \in
E$, $(u,v)$ and $p_{uv}$ are unknown.  We will describe the version
for the active model here and refer to the Previous Results section 
for the description of the passive model.  Under the active query model, 
for each round $i=1,2,\ldots,M$, the algorithm can perform a query 
by choosing a seed set $S_i \subseteq V$.  The contagion process 
described above then starts from $S_i$ and returns the set of 
infected vertices $A_i$.  The goal of the problem is to find an algorithm 
that uses a small number of rounds $M$, and correctly returns the edge set $E$.

{\bf Previous Results.}  Amin, Heidari, and Kearns~\cite{AHK2014}
considered the problem in both active model and passive model.  They
proved that tree networks can be exactly learned in both models.  In
addition, they also considered the problem for learning non-tree
networks.

Since our focus is on the active model, we start by describing their
algorithm for learning trees in this model, later referred to as the
AHK algorithm.  For any vertex $u \in V$, the algorithm repeatedly
queries with the seed set containing only a single vertex $u$ in order
to infer the set of vertices $\Gamma(u)$ adjacent to $u$.  Let $R_u(v)$ 
be the set of rounds that $v$ becomes infected while $\{u\}$ be the 
seed set, i.e., $R_u(v) = \{i : v \in A_i \mbox{ and } S_i = \{u\}\}$.  
The algorithm infers that $u$ and $v$ are adjacent if and only if there
does not exist a vertex $w \in V \backslash \{u,v\}$ such that $R_u(v)
\subseteq R_u(w)$.  The algorithm needs
$O(\frac{n}{\Delta^2}\log\frac{n}{\delta})$ queries to learn the tree
with probability at least $1-\delta$.
We note that the AHK algorithm could fail when applying to non-tree networks
because a vertex can be infected from the seed set through many
possible paths.

For the passive model, the seed sets are chosen randomly from a
distribution where each vertex is chosen independently.  The algorithm
presented by Amin {\em et al.} for this model employs the lift
function $L(v|u)$ which is the increase in the probability that vertex
$v$ becomes infected when $u$ is in the seed set. The algorithm finds
an estimate $\hat{L}$ of $L$, and if $\hat{L}$ is close to $L$, they
showed that the algorithm can exactly learn the tree.

For non-tree networks, Amin {\em et al.} presented an algorithm,
called the $K$-lifts algorithm, which can be viewed as a generalized
version of the algorithm for learning trees in the passive model.
Given the number of edges $K$, the algorithm returns $K$ pairs of
vertices with highest estimated lift scores as network edges.  The
experimental results showed that the $K$-lifts algorithm performs well
when learning sparse random networks (under the Erd\H{o}s-R\'{e}nyi
model~\cite{ER1959, G1959} and the Small-World model~\cite{WS1998}).
On the positive side, Amin {\em et al.} proved that the $K$-lifts 
algorithm can learn cycle networks.  However, they showed that 
the $K$-lifts algorithm fails to learn a network $H$ on $2n-1$ vertices 
constructed by joining a star with $n$ vertices rooted at vertex $v_0$ 
and an $n$-cycle containing $v_0$ at $v_0$.

{\bf Our Results.}  Here we state our results formally. Our first
result considers large girth networks.  For a network $G=(V,E)$, the
\textit{girth} $g$ of $G$ is the length of the shortest cycle in the
network.  We also need another property related to the number of
simple paths.  Denote by $\mathcal{P}(u,v)$ the set of simple paths
from vertex $u \in V$ to vertex $v \in V$ in $G$, and
$\mathcal{P}_d(u,v)$ the set of paths of length $d$ in
$\mathcal{P}(u,v)$.  Let $p_d$ be the maximum number of simple paths
of length $d$ between any pair of vertices, i.e., $p_d = \max_{u \in
  V, v \in V}|\mathcal{P}_d(u,v)|$.
We define the {\em path growth rate} $\rho = \max_{d}(p_d)^{1/d}$.
The parameter $\rho$ intuitively represents the growth rate for the
number of simple paths in the network.  Note that for tree networks,
$g$ can be regarded as $\infty$ and $\rho = 1$.  

We show that if $\rho(1-\Delta) < 1$ and $g>f(\Delta,\rho)$, for some
function $f$ (stated in
Theorem~\ref{thm:label_complexity_for_large_girth_network}), we can
successfully learn the network in the active model with
$O(\frac{n}{\Delta^2}\log\frac{n}{\delta})$ queries with probability
at least $1 - \delta$.  We note that our algorithm can successfully
learn the Amin {\em et al.}'s counter example $H$ (discussed in
Theorem 6 in~\cite{AHK2014}) since the girth of $H$ is $n$ and its
path growth ratio is $2^{2/n}$, which is close to $1$.

The algorithm requires $O(\frac{n}{\Delta^2}\log\frac{n}{\delta})$
active queries, which is the same bound as the AHK algorithm of
Amin~{\em et al.} under the same model.  While the bound itself does
not depend on the values of $\rho$ and $g$, our proofs of correctness
require the network to satisfies certain conditions on $\Delta$, $\rho$
and $g$ (see Theorem~\ref{thm:label_complexity_for_large_girth_network}).

The second result is on the bounded-degree networks.  We show that, if
the maximum degree of the network is $D$, in the active model, these
networks can be exactly learned by a very simple algorithm that makes
at most $O(\frac{n}{\Delta^{2D}} \log \frac{n}{\delta})$ queries with
probability at least $1-\delta$.


\section{Learning Large Girth Networks}
\label{section:large_girth}

This section describes an algorithm that learns large-girth networks.
We start with a crucial lemma on the properties of these networks.  As
in the AHK algorithm, we would like to filter out non-adjacent pairs
of vertices.  We focus on pairs of vertices that are close, but not
adjacent.  Let $d_{uv}$ be the shortest path distance from $u$ to $v$.

The following lemma is a key observation.

\begin{lemma}\label{lemma:unique_short_stp}
Let $G=(V,E)$ be a network with girth $g$. For any pair of vertices
$u, v \in V$ such that $1<d_{uv}<g/2$, there is a unique shortest path
from $u$ to $v$, and all other paths from $u$ to $v$ have length 
greater than $g/2$.
\end{lemma}

\begin{proof}
  We prove by contradiction. Let $P_1$ be the shortest path from $u$
  to $v$ of length $k_1 < g/2$ and let $P_2$ be another path from $u$
  to $v$ of length $k_2$ where $k_1 \leq k_2 \leq g/2$.  We can take a
  union of $P_1$ and $P_2$ and obtain a cycle of length at most $k_1 +
  k_2 < g$, contradicting the fact that $G$ has girth $g$.
\end{proof}

Using Lemma~\ref{lemma:unique_short_stp} with appropriate value of
$g$, we can show that it is very unlikely that, when $\{u\}$ is the
seed set, $v$ is infected through paths other than the unique shortest
path from $u$.  This implies that in most rounds when $v$ is infected,
every intermediate vertex $w$ in the shortest path from $u$ to $v$
must be infected as well.

Recall that $R_u(v)$ be the set of rounds that $v$ becomes infected
while $\{u\}$ be the seed set.  In a tree network, the rejection
criteria of the AHK algorithm works because $R_u(v)\subseteq R_u(w)$
for any intermediate vertex $w$ in the shortest path from $u$ to $v$.
However, in general, since $v$ can be infected through other paths, we
need a milder criteria.  Instead of requiring that $R_u(v)\subseteq
R_u(w)$ to reject $(u,v)$, we shall reject $(u,v)$ as an edge when
there exists a vertex $w$ such that $w$ appears too often with $v$,
i.e., when the set $R_u(v)\cap R_u(w)$ is large.

Let $m$ be the number of rounds that we query for a single seed set
(to be specified later).  Our modification of the AHK algorithm to
learn large girth networks is shown in Algorithm~\ref{alg:girth}.
Note that although the contagion parameter $\Delta$ is required to
make decision in Algorithm~\ref{alg:girth}, it is enough to use its
lower bound. The AHK algorithm itself does not require $\Delta$, but
the parameter is implicitly needed to make sure that the number of
queries is large enough.

\algsetup{indent=2em}
\begin{algorithm}[h!]
\caption{Algorithm for learning a large girth network $G=(V,E)$}\label{alg:girth}
\begin{algorithmic}
\STATE $E' \leftarrow \emptyset$
\FORALL{$u$ in $V$}
	\FOR{$i=1$ to $m$}
		\STATE query with seed set $S_i = \{u\}$, then receive the set of infected vertices $A_i$
	\ENDFOR
	\FORALL{$v$ in $V \backslash \{u\}$}
		\IF{$\forall w \in V \backslash \{u,v\}$, $|R_u(v) \backslash R_u(w)| > \frac{3\Delta^2m}{8}$}
			\STATE $E' \leftarrow E' \cup \{(u,v)\}$
		\ENDIF
	\ENDFOR
\ENDFOR
\RETURN $E'$.
\end{algorithmic}
\end{algorithm}

We shall show that the Algorithm~\ref{alg:girth} returns edges $E$ with high
probability after querying $M = nm$ rounds in total, if $m$ is large enough. 

We would like to point out that our algorithm works only when
$\rho(1-\Delta) < 1$.  This bound is essential for preventing issues
that may occur when high growth rate compensates the infectious
failure based on our analysis techniques.  See a discussion at the end
of Lemma~\ref{lemma:prob_active_path_length_k}.

After each round of query we say that path $P$ is \textit{active} if
every edge in $P$ is active.  (Note that each edge must be active in
the right direction.)  On the other hand, $P$ is \textit{inactive} if
there exists an inactive edge in the path.
  
The next lemma shows that if $\rho(1-\Delta) < 1$, the probability
that there is an active path of length $k$ from vertex $u$ to vertex
$v$ depends on $\rho(1-\Delta)$.

\begin{lemma}\label{lemma:prob_active_path_length_k}
  For any network $G=(V,E)$, if parameter $\rho$ of $G$ satisfies the
  condition that $\rho < 1/(1-\Delta)$, then for any vertex $u \in V$
  and vertex $v \in V$, the probability that $u$ infects $v$ along any
  paths of length at least $k$ is at most
  $\frac{(\rho(1-\Delta))^k}{1-\rho(1-\Delta)}$.
\end{lemma}
\begin{proof}
Using the union bound, the probability that $u$ infects $v$ along any
paths of length at least $k$ is at most
\begin{eqnarray*}
\sum_{d=k}^{n-1}{(|{\mathcal P}_d(u,v)| \times \beta^d})	
			&\leq& \sum_{d=k}^{n-1}{ \rho^d \beta^d } 
			\leq \sum_{d=k}^{n-1}{ (\rho(1-\Delta))^d } \\
			&=& \frac{ (\rho(1-\Delta))^k - (\rho(1-\Delta))^n }{ 1-\rho(1-\Delta) } 
			\leq \frac{ (\rho(1-\Delta))^k }{ 1-\rho(1-\Delta) }.
\end{eqnarray*}
Note that we use the fact that $\rho(1-\Delta)<1$ in the last inequality.
\end{proof}

The requirement that $\rho(1-\Delta)<1$ is essential to ensure that
the sum $\sum_{d=k}^{n-1}\rho^d\beta^d$ converges nicely even when 
$n$ is large.  Note that when the requirement is not true, the contagion
process starting at a single seed vertex can reach a vertex very far from
the seed.  To see this, take a perfect $k$-any tree with $L$ levels.
The contagion process starting at the root resembles the branching
process where the offspring distribution is binomial with parameter
$k$ and $p$, where $p$ is the infectious probability.  It is known
that if the infectious probability is $1/k+\epsilon$, it is very
likely that some leaf on the $L$-th level will be infected.  Since our
analysis uses a simple union bound that neglects dependencies between
paths, it fails to distinguish this situation with the one where a lot
of leaves are infected, and finally it fails to show that the probability
that a particular node far away from the seed becomes infected is very small.

From Lemma~\ref{lemma:prob_active_path_length_k}, we have the next
corollary that provides the lower bound on the girth so that the
probability of having long active paths is at most $\Delta^2/4$.  The
ceiling function appears because
Lemma~\ref{lemma:prob_active_path_length_k} works only when $k$ is an
integer, and as a by-product, that the lower bound on $g$ is even.

\begin{corollary}\label{corollary:exist_long_path_prob_is_small}
For a contagion process with parameter $\Delta$ over a network
$G=(V,E)$, if the girth $g$ of $G$ and the path growth rate $\rho$
satisfy the following inequalities:
\begin{equation} \label{eq:growth_bounded}
	1 \leq \rho < \frac{1}{1-\Delta}
\end{equation}
\begin{equation} \label{eq:girth_bounded}
	g \geq 2\left\lceil\frac{ 2\log\frac{\Delta}{2} + \log(1-\rho(1-\Delta)) } { \log \rho(1-\Delta) }\right\rceil
\end{equation}
then for any pair of vertices $u \in V$ and $v \in V$, the probability
that there exists an active path of length at least $g/2$
between $u$ and $v$ is at most $\Delta^2/4$.
\end{corollary}

We shall use the bound of $g$ in the previous corollary as the lower bound
of the girth.  Note that the lower bound is not extremely large.  When
$\Delta=1/2$ and $\rho=1.25$, the algorithm works when $g\geq 16$.
When $\Delta=1/2$ and $\rho=1.5$, $g\geq 30$.

Later on in this section, we assume that we work on the network whose
parameter $\rho$ and girth $g$ satisfy inequalities
(\ref{eq:growth_bounded}) and (\ref{eq:girth_bounded}), respectively.
Moreover, for technical reasons (see the proof of
Lemma~\ref{lemma:include_adjacent_vertices}), we also assume w.l.o.g.
that $g$ is even.  When the girth $g$ of the network is odd but
satisfies condition (2) from
Corollary~\ref{corollary:exist_long_path_prob_is_small}, we can take
$g'=g-1$ as its lower bound on the girth and apply the results.

Our main theorem shows that for any network $G=(V,E)$ in this class,
the Algorithm~\ref{alg:girth} returns the edges $E$ with high
probability.  To prove the theorem, we need the following 3 lemmas.

The following lemma deals with the case that $(u,v)$ is an edge in the
network.

\begin{lemma}\label{lemma:include_adjacent_vertices}
Assume that the network does not have any cycle of length shorter than
$g$, when $g$ is even, and all the network parameters satisfy the
conditions in Corollary~\ref{corollary:exist_long_path_prob_is_small}.
For any pair of adjacent vertices $u,v \in V$ and any vertex $w\in
V\setminus\{u,v\}$, in any round $i$ where the algorithm queries with
seed set $S_i=\{u\}$, the probability that $v\in A_i$, but $w\notin
A_i$ is at least $7\Delta^2 / 8$.  Thus, the expected size of $R_u(v)
\backslash R_u(w)$ is at least $7\Delta^2m / 8$.
\end{lemma}
\begin{proof}
We first analyze the probability. There are two cases.

{\em Case 1:} $v$ is in a shortest path from $u$ to $w$. Let $P$ be
the shortest path from $u$ to $w$ containing $v$.  Note that since
$(u,v)\in E$, edge $(u,v)$ is the first edge in the path.  Let $e$ be
an edge in $P$ adjacent to edge $(u,v)$, i.e., $e$ is the next edge
after $(u,v)$ in $P$.

Let $\mathcal{A}$ be the event that $(u,v)$ is active, $\mathcal{B}$
be the event that $e$ is active and $\mathcal{C}$ be the event that
there exists an active path in ${\mathcal P}(u,w)$.  Note that when
event $\mathcal{A}\cap\overline{\mathcal B}\cap\overline{\mathcal C}$
occurs, we know that $v \in A_i$ and $w \notin A_i$.  Thus, we have
\begin{eqnarray*}
\Pr[v \in A_i,w \notin A_i] &\geq&	\Pr[\mathcal{A} \cap \overline{\mathcal{B}} \cap \overline{\mathcal{C}}] = \Pr[\mathcal{A}] \times \Pr[\overline{\mathcal{B}} | \mathcal{A}] \times \Pr[\overline{\mathcal{C}} | \mathcal{A} \cap \overline{\mathcal{B}}] \\
                             &\geq& \Delta^2 \times (1 - \Pr[\mathcal{C} | \mathcal{A} \cap \overline{\mathcal{B}}])
\end{eqnarray*}
The last inequality holds because $\mathcal{A}$ and $\mathcal{B}$ are
independent, and each occurs with probability at least $\Delta$.

We are left to compute the upper bound of the probability of the event
$\mathcal{C}$ given $\mathcal{A} \cap \overline{\mathcal{B}}$.  The
condition $\mathcal{A} \cap \overline{\mathcal{B}}$ implies that $P$,
which is a shortest path, is inactive.  There are two possible
subcases that $w$ can be infected: (i) from a path $P'$ in ${\mathcal
  P}(u,w)$ that uses edge $(u,v)$ or (ii) from a path $P''$ in
${\mathcal P}(u,w)$ that does not use $(u,v)$.

Let's consider subcase (i) first.  Since $P$ is a shortest path; the
postfix $P_v$ of $P$ starting at $v$ ending at $w$ is also a shortest
path from $v$ to $w$.  Let $P'_v$ be the postfix of $P'$ starting at
$v$.  We claim that $P'_v$ is of length at least $g/2$.  This is true
when the shortest path $P_v$ is of length at least $g/2$.  Thus,
assume otherwise, i.e., $P_v$ is of length less than $g/2$; applying
Lemma~\ref{lemma:unique_short_stp}, we have that all other paths from
$v$ to $w$ are of length greater than $g/2$ as required.  Since the 
paths are long, Corollary~\ref{corollary:exist_long_path_prob_is_small} 
implies that the probability that we have an active path is at most 
$\Delta^2/4\leq 1/16$.

Next, consider subcase (ii).  Using the same argument as in subcase (i),
we have that $P''$ is of length at least $g/2$; thus, the probability 
that we have an active path in this case is at most $\Delta^2/4\leq 1/16$.

Combining these two subcases using the union bound, we have that
$\Pr[\mathcal{C}|\mathcal{A} \cap\overline{\mathcal{B}}]\leq
\Delta^2/4+\Delta^2/4 \leq 1/8$.  Therefore, the probability that $v
\in A_i$ and $w \notin A_i$ is at least $7\Delta^2/8$.

{\em Case 2:} $v$ is not in any shortest paths between $u$ and $w$.
Let $P$ be a shortest path from $u$ to $w$ and $e$ be an edge in $P$
that is adjacent to $u$ (i.e., $e$ is the first edge in $P$).
Provided that $(u,v)$ is active but $e$ is inactive, $w\in A_i$ only
when there exists an active path in $\mathcal{P}(u,w)$.  Again, let
${\mathcal A}$ be the event that $(u,v)$ is active, $\mathcal{B}$ be
the event that $e$ is active, and $\mathcal{C}$ be the event that
there exists an active path in ${\mathcal P}(u,w)$.  As in the
previous case, we have that
$ \Pr[v \in A_i,w \notin A_i] \geq \Pr[\mathcal{A} \cap \overline{\mathcal{B}} \cap \overline{\mathcal{C}}] $.

We focus on the event ${\mathcal C}$ given
$\mathcal{A}\cap\overline{\mathcal B}$.  If an active path $P'$ in
${\mathcal P}(u,w)$ uses edge $(u,v)$, it has to be of length greater
than $g/2$ because $v$ is not in any shortest paths from $u$ to $w$.
Since $g$ is even, $g/2$ is an integer; thus, the length of $P'$ is at
least $g/2+1$.  This implies that the postfix of $P'$ starting at $v$
is of length at least $g/2$.  From
Corollary~\ref{corollary:exist_long_path_prob_is_small}, the
probability that there exists an active path in this case is at
most $\Delta^2/4\leq 1/16$.  On the other hand, as $e$ is inactive, an
active path from $u$ to $w$, not going through $(u,v)$, has length at
least $g/2$.  Again,
Corollary~\ref{corollary:exist_long_path_prob_is_small} implies that
the probability of this case is at most $\Delta^2/4\leq 1/16$.  With
the union bound, we have that $\Pr[\mathcal{C} | \mathcal{A} \cap
  \overline{\mathcal{B}}]\leq 1/16 + 1/16 = 1/8$.  Hence, $\Pr[v \in
  A_i,w \notin A_i]$ is at least
\[
\Pr[\mathcal{A} \cap \overline{\mathcal{B}} \cap \overline{\mathcal{C}}]
=
\Pr[\mathcal{A}] \times \Pr[\overline{\mathcal{B}} | \mathcal{A}] \times \Pr[\overline{\mathcal{C}} | \mathcal{A} \cap \overline{\mathcal{B}}]
\geq \Delta^2 \times (7/8) = \frac{7\Delta^2}{8}.
\]

Since for both cases, the probability $\Pr[v \in A_i,w \notin A_i]$ is
at least $7\Delta^2/8$ and the algorithm makes $m$ rounds of
queries with seed set $\{u\}$, the expected size of $R_u(v)
\backslash R_u(w)$ is at least $7\Delta^2m/8$, as required.
\end{proof}

The next two lemmas consider non-adjacent vertices $u$ and $v$.  When
$u$ is close to $v$, we use Lemma~\ref{lemma:exclude_close_vertices};
otherwise, we use Lemma~\ref{lemma:exclude_far_vertices}, whose proof
uses Corollary~\ref{corollary:exist_long_path_prob_is_small} and is
omitted to save space.

\begin{lemma}\label{lemma:exclude_close_vertices}
For any pair of non-adjacent vertices $u, v \in V$ where $d_{uv} < g/2$, 
there exists a vertex $w \in V \backslash \{u,v\}$ such that in any round $i$ 
where the algorithm queries with the seed set $S_i = \{u\}$, the probability 
that $v \in A_i$, but $w \notin A_i$ is at most $\Delta^2/4$. Thus, the expected 
size of $R_u(v) \backslash R_u(w)$ is at most $\Delta^2m/4$.
\end{lemma}
\begin{proof}
From Lemma~\ref{lemma:unique_short_stp}, there is only one shortest
path from $u$ to $v$.  Let $w$ be the second vertex in the shortest
path from $u$ to $v$.  Consider the case that $v \in A_i$ but $w
\notin A_i$. In this case, the edge $(u,w)$ must be inactive.  This
implies that the shortest path from $u$ to $v$ is also inactive, thus
the seed $u$ infects $v$ along a non-shortest path. From
Lemma~\ref{lemma:unique_short_stp}, any non-shortest paths is of
length greater than $g/2$.  Using
Corollary~\ref{corollary:exist_long_path_prob_is_small}, we have that
$\Pr[v \in A_i, w \notin A_i] \leq \Delta^2/4$. Since the algorithm
makes $m$ rounds of queries with seed set $\{u\}$, the expected size
of $R_u(v) \backslash R_u(w)$ is at most $\Delta^2m/4$.
\end{proof}

\begin{lemma}\label{lemma:exclude_far_vertices}
For any pair of non-adjacent vertices $u, v \in V$ where $d_{uv} \geq
g/2$, in any round $i$ where the algorithm queries with the seed set
$S_i = \{u\}$, the probability that $v \in A_i$ is at most
$\Delta^2/4$.  Thus, the expected size of $R_u(v)$ is at most
$\Delta^2m/4$.
\end{lemma}

Lemma~\ref{lemma:exclude_far_vertices} implies that for any vertex
$w$, the expected size of $R_u(v)\setminus R_u(w)$ is at most
$\Delta^2m/4$.  The previous 3 lemmas show the expectation gap between
$7\Delta^2m/8$ and $\Delta^2m/4$ of the size of $R_u(v)\setminus
R_u(w)$ for some vertex $w$.  Using the Chernoff bound, we can prove
the main theorem.  Its proof is omitted to save space.

\begin{theorem}\label{thm:label_complexity_for_large_girth_network}
Suppose network $G=(V,E)$ has the parameter
\[
	\rho \in [1,\frac{1}{1-\Delta}) \mbox{ and }
	g \geq 2\left\lceil\frac{2\log\frac{\Delta}{2} +
    \log(1-\rho(1-\Delta))} { \log \rho(1-\Delta) }\right\rceil.
\]	
  The Algorithm~\ref{alg:girth} returns the edges $E$ with probability at
  least $1-\delta$ after querying at most $O(\frac{n}{\Delta^2}\log\frac{n}{\delta})$ rounds.
\end{theorem}

\section{Learning Bounded Degree Networks}
\label{section:bounded_degree}

This section shows that if the maximum degree of a network is $D$, we
can recover all edges of the network with probability at least
$1-\delta$ using polynomial queries in term of $n$, $1/\Delta$
and $1/\delta$.  The key idea is that if the maximum degree of
the network is bounded, we could observe a single edge from the
results of queries. The algorithm is fairly straight-forward.  For each vertex $u
\in V$, the algorithm repeatedly selects $\{u\}$ to be the seed set
for $m$ rounds, where
$m=O(\frac{1}{\Delta^{2D}}\log\frac{n}{\delta})$.  For any vertex $v
\in V$, the algorithm includes $(u,v)$ to the set $E'(u)$, if there
exists round $i$ such that $S_i = \{u\}$ and $A_i = \{u, v\}$.  After
receiving all results of $nm$ queries, the algorithm returns $E' =
\cup_{u \in V}E'(u)$.

\begin{theorem}\label{thm:label_complexity_for_bounded_degree_network}
Let $G=(V,E)$ be a network with maximum degree $D$.  The algorithm
described above can return the edges $E$ with probability at least
$1-\delta$ by querying at most $O(\frac{n}{\Delta^{2D}} \log
\frac{n}{\delta})$ rounds.
\end{theorem}
\begin{proof}
Since the algorithm will not include edges not in $E$, we consider
the probability that the algorithm recovers all edges in $E$.

Consider edge $(u,v) \in E$.  Consider the round $i$ where $\{u\}$ is the seed
set, i.e., $S_i=\{u\}$.  The probability that we observe only edge
$(u,v$), i.e., $A_i=\{u,v\}$, is
\[
p_{uv} \times \prod_{w \in \Gamma(u) \backslash \{v\}}{(1-p_{uw})} \times \prod_{w \in \Gamma(v) \backslash \{u\}}
{(1-p_{vw})} \geq	\Delta^{2D},
\]
where $\Gamma(u)$, for any $u\in V$, is a set of neighbors of $u$.

Since we perform $m$ rounds of queries for $u$, the probability that
we fail to observe edge $(u,v)$, when the seed set is $\{u\}$, is at
most
\[
(1-\Delta^{2D})^m \leq \exp(-m\Delta^{2D}).
\]
If we let $m=O(\frac{1}{\Delta^{2D}} \log \frac{n}{\delta})$, the
failure probability is at most $\delta/n^2$.  Using the union bound,
the probability that the algorithm fails to observe any edge is at
most $|E|\cdot\delta/n^2\leq\delta$.
\end{proof}

\section*{Acknowledgements}
We would like to thank anonymous reviewers for their very 
helpful comments.
We gratefully acknowledge the Thailand Research Fund (TRF) 
for financial support through the Royal Golden Jubilee (RGJ)
Ph.D. Programme under the Grant No. PHD/0310/2550.

\bibliographystyle{plainnat}

\end{document}